\documentclass[conference]{IEEEtran}
\IEEEoverridecommandlockouts
\usepackage{cite}
\usepackage{amsmath,amssymb,amsfonts}
\usepackage{algorithmic}
\usepackage{graphicx}
\usepackage{textcomp}
\usepackage{xcolor}
\def\BibTeX{{\rm B\kern-.05em{\sc i\kern-.025em b}\kern-.08em
    T\kern-.1667em\lower.7ex\hbox{E}\kern-.125emX}}

\usepackage{amsthm}
\usepackage{makecell}
\usepackage{subcaption}

\newcommand\methodName{\textit{MiD}}

\makeatletter
\newcommand{\linebreakand}{
  \end{@IEEEauthorhalign}
  \hfill\mbox{}\par
  \mbox{}\hfill\begin{@IEEEauthorhalign}
}
\makeatother

\usepackage{xpatch}
\makeatletter
\xpatchcmd{\@thm}{\thm@headpunct{.}}{\thm@headpunct{}}{}{}
\makeatother

\newtheorem{prop}{Proposition}

\begin{document}

\title{Power of Explanations: Towards automatic debiasing in hate speech detection
% Automatically debiasing text classifiers
% { \footnotesize \textsuperscript{*}Note: Sub-titles are not captured in Xplore and should not be used}
\thanks{
% Yi Cai is supported by the ZVKI funded by BMJV and the ResponsibleAI program funded by Nds. MWK during the period of completing this work.
Yi Cai and Gerhard Wunder are supported by the Center of Trustworthy AI (www.zvki.de) funded by the Federal Ministry of Enviroment, Nature Conservation, Nuclear Safety and Consumer Protection (BMUV).
Gerhard Wunder is also supported by the German excellence cluster 6GRIC (6g-ric.de) funded by the Federal Ministry of Education and Research (BMBF) as well as German Science Foundation (DFG) priority programs under the grants WU 598/7-2, WU 598/8-2 and WU 598/9-2.

Yi Cai was also supported by the ResponsibleAI program funded by Nds. MWK during the period of completing this work. 
}
}

\author{
\IEEEauthorblockN{Yi Cai}
% \IEEEauthorblockA{\textit{Dept. of Mathematics and Computer Science} \\
\IEEEauthorblockA{\textit{Dept. of Math. and Comp. Science} \\
\textit{Freie Universität Berlin}\\
Berlin, Germany \\
yi.cai@fu-berlin.de}
\and
\IEEEauthorblockN{Arthur Zimek}
\IEEEauthorblockA{\textit{Dept. of Math. and Comp. Science} \\
\textit{University of Southern Denmark}\\
Odense, Denmark \\
zimek@imada.sdu.dk}
\linebreakand
\IEEEauthorblockN{Gerhard Wunder}
\IEEEauthorblockA{\textit{Dept. of Math. and Comp. Science} \\
\textit{Freie Universität Berlin}\\
Berlin, Germany \\
gerhard.wunder@fu-berlin.de}
\and
% \IEEEauthorblockN{Eirini Ntoutsi\thanks{Eirini Ntoutsi was affiliated with Freie Universität Berlin where most of the work was carried out.}} % \IEEEauthorrefmark{1}
\IEEEauthorblockN{Eirini Ntoutsi\IEEEauthorrefmark{1}\thanks{\IEEEauthorrefmark{1}Eirini Ntoutsi was affiliated with Freie Universität Berlin where most of the work was carried out.}}
\IEEEauthorblockA{\textit{Research Institute CODE} \\
\textit{Universität der Bundeswehr München}\\
Munich, Germany \\
eirini.ntoutsi@unibw.de}
}

\maketitle

\begin{abstract}
% This document is a model and instructions for \LaTeX.
% This and the IEEEtran.cls file define the components of your paper [title, text, heads, etc.]. *CRITICAL: Do Not Use Symbols, Special Characters, Footnotes, 
% or Math in Paper Title or Abstract.
Hate speech detection is a common downstream application of natural language processing (NLP) in the real world.
In spite of the increasing accuracy, current data-driven approaches could easily learn biases from the imbalanced data distributions originating from humans.
The deployment of biased models could further enhance the existing social biases.
But unlike handling tabular data, defining and mitigating biases in text classifiers, which deal with unstructured data, are more challenging.
A popular solution for improving machine learning fairness in NLP is to conduct the debiasing process with a list of potentially discriminated words given by human annotators.
In addition to suffering from the risks of overlooking the biased terms, exhaustively identifying bias with human annotators are unsustainable since discrimination is variable among different datasets and may evolve over time.
To this end, we propose an automatic misuse detector (MiD) relying on an explanation method for detecting potential bias.
And built upon that, an end-to-end debiasing framework with the proposed staged correction is designed for text classifiers without any external resources required.
\end{abstract}

\begin{IEEEkeywords}
% component, formatting, style, styling, insert
AI fairness, bias detection, bias mitigation, explainable AI, text classification
\end{IEEEkeywords}

\section{Introduction} \label{sec:intro}
Although the recent breakthrough led by attention mechanism~\cite{vaswani2017attention} is beneficial to the increasing performance in downstream tasks~\cite{li2022survey}, the growing complexity of models leads to more concerns about machine learning fairness due to the opacity~\cite{wallace2019allennlp, muller2020ethics}.
Previous work shows that bias widely exists in various corpora for natural language processing (NLP)~\cite{dixon2018measuring, schick2021self} and can be easily learned by even the most advanced text classifiers~\cite{kennedy2020contextualizing}.
% Hate speech detection as one of the downstream tasks has been widely applied on social media platforms~\cite{TOD1}.
Hate speech detection as one of the downstream tasks has been widely applied on social media platforms.
Biases held by these detectors will certainly harm the right of specific groups to be referred to or express themselves.
Solving bias in this scenario is therefore crucial. 

With years of discussion on machine learning fairness, the majority puts their efforts into defining and mitigating bias in classifiers for tabular data.
But the unstructured nature of textual data forbids the direct use of numerous debiasing methods~\cite{iosifidis2019fae, iosifidis2019adafair, hardt2016equality} specialized for structured data.
A popular solution for identifying bias in NLP is to manually select a list of words from the given vocabulary, which usually refer to demographic information (for example, gender and ethnicity), then measure the performance difference under similar contexts for various groups as bias~\cite{kiritchenko2018examining, garg2019counterfactual}.
Having the bias defined, the improvement of text classifiers in terms of fairness is feasible through different approaches, such as instance weighting~\cite{zhang2020demographics}, data augmentation~\cite{rudinger2018gender}, and feature attribution suppression~\cite{kennedy2020contextualizing, yao2021refining}.
However, using a manually defined list can again introduce discrimination to the debiasing process by either under-/over-representing demographic groups in the list as shown in Fig.~\ref{fig:attr_example}.
\begin{figure}
    \centering
    \includegraphics[width=0.4\textwidth]{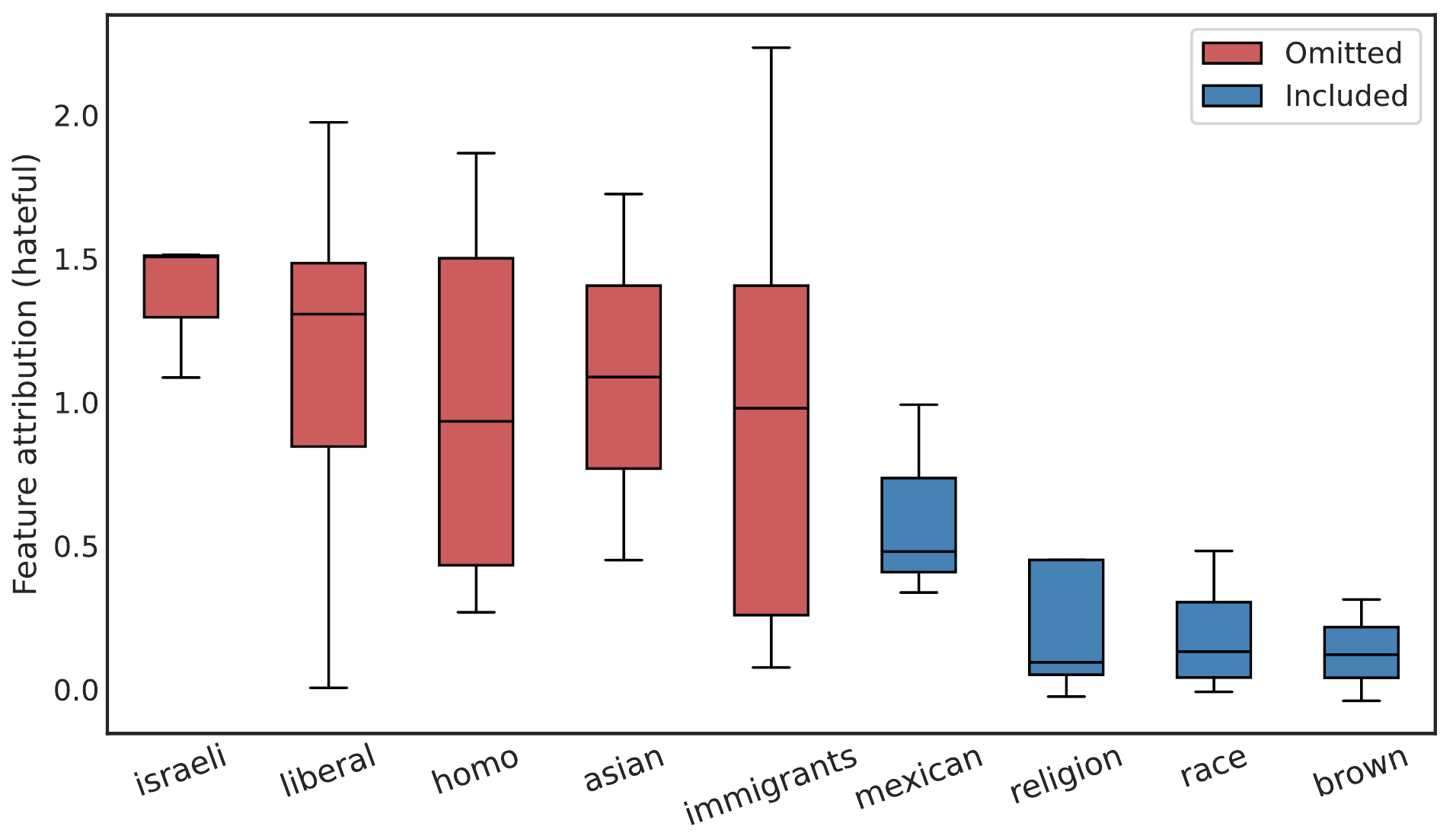}
    % \includesvg[width=0.4\textwidth]{pic/attr_example.svg}
    \caption{Feature attribution extracted from a fine-tuned BERT model without debiasing for words omitted/included in the manual list used in~\cite{kennedy2020contextualizing}.
    The omitted neutral group identifiers actively contributing to the prediction as hateful, which indicates strong biases, are overlooked by human annotators. 
    On the contrary, some less biased words are selected for the debiasing list.
    }
    \label{fig:attr_example}
\end{figure}
While omitting discriminated group identifiers can undoubtedly strengthen specific biases, over-representation in the debiasing list could also limit model performance through indirect impacts as demonstrated by our experimental results.
Moreover, discrimination learned by models is variable to datasets, hyperparameters, and training processes.
Hence, exhaustively defining bias through human annotators is unsustainable.

To this end, we proposed a fully automatic \underline{mi}suse \underline{d}etector (\methodName), which deploys an explanation method, to identify potentially biased features (tokens/words) at the training stage without any external knowledge or resources.
Aiming to perform the detection at runtime, we introduced the false positive proportion \textit{FPP} as an efficient proxy of feature contributions to predictions since the time complexity of feature attribution examination with the explanation method is impractical.
Built upon that, an end-to-end debiasing framework\footnote{The source code is available at https://github.com/caiy0220/PoE} is designed to mitigate bias in hate speech classifiers during the training process.
% We showed in the experiments that the list of potentially biased words delivered by the proposed \methodName~has good coverage on the manually defined list and results in an outstanding performance in balancing accuracy and fairness.
The experimental results show that the list of potentially biased words delivered by the proposed \methodName~has good coverage on the manually defined list and results in an outstanding performance in balancing accuracy and fairness.
Another key finding is that the correction targeting single words has indirect impacts on their semantic neighbors, which calls for more attention to potential effects the debiasing process could bring into hate speech detectors.

The rest of the paper is organized as follows. In Section~\ref{sec:related}, we discuss the related work of machine learning fairness and current progress in explainable AI (XAI), which is an essential component of the proposed framework.
Section~\ref{sec:mid} details the proposed method \methodName~for automatic bias detection.
Afterwards, the debiasing framework with a staged correction is introduced in Section~\ref{sec:staged}.
To evaluate the proposed method, we conduct detailed experiments in Section~\ref{sec:exp}.
Finally, we discuss and conclude the findings of this work in Section~\ref{sec:concl}.
% Finally, we discuss the findings of this work and the future research directions in Section~\ref{sec:concl}.

% A word could be frequently observed in false positive samples, but that is not necessary for it to be biased.

\section{Related work} \label{sec:related}
To mitigate discrimination in hate speech classifiers, there are debiasing methods at different stages of the machine learning pipeline proposed.
Data augmentation tackles the problem at the data preprocessing stage.
Instances relating to under-represented groups are augmented through random combinations of pre-defined templates and group identifiers~\cite{zhao2018gender, kiritchenko2018examining, dixon2018measuring}.
But augmentation without supervision always delivers meaningless and unrealistic instances, which could introduce potential risks into the system.
Similar to the previous solution, instance weighting balances the distribution of class labels over ethnic groups by editing the weights of entries in the training set~\cite{zhang2020demographics} during the preprocessing.
Without creating synthetic samples, it avoids the potential risks.

Our work is highly related to the recently emerging work which mitigates bias at the training stage.
The idea is to debias hate speech classifiers during training by suppressing unwanted model attention on selected neutral words~\cite{kennedy2020contextualizing}.
A static list of group identifiers gives the definition of the sensitive neutral words.
Instead of preparing a pre-defined list, \cite{yao2021refining} suggests correcting model behaviors by involving human-in-the-loop.
For both methods, the essential part is the employment of an explanation method, which dominates the derivation of feature attribution.

A feature/word is not necessary to be biased even if it frequently appears in incorrect predictions.
Therefore, obtaining insights into the decision making process is the key to the precise treatment for its bias.
Recent developments in XAI, especially in explaining text classifiers~\cite{chen2020generating, murdoch2018beyond, cai2021xproax}, offer a solution to the detailed analysis by revealing the evidence supporting a decision through explanations.
Among different categories, model-agnostic local explanation methods~\cite{ribeiro2016should, lundberg2017unified} are in favor with debiasing text classifiers,
as methods belonging to this category would not harm the flexibility of the debiasing framework since no prerequisites are present for the targeting model, 
In addition, because state-of-the-art text classifiers treat the same token differently depending on the context~\cite{peters2017semi, sarzynska2021detecting}, we prefer a local explanation method that interprets one decision at a time rather than a global one.
Regarding the concrete form of presenting explanations, feature importance~\cite{vstrumbelj2014explaining} outperforms other choices, such as saliency map~\cite{smilkov2017smoothgrad} and counterfactual~\cite{naumann2021consequence}, as it quantifies feature contributions to a prediction, which is accessible to both machine and human.
% \textbf{Debiasing text classifiers} 
% Bias in corpus for textual data --> widely observed --> Imbalance distribution --> Consequence
% Solution: template, weighting, attribution suppression

% \textbf{Explainable AI}
% Understanding/Explaining AI models draws attention --> many efforts have been put on --> presented through feature attribution, saliency map, counterfactuals~\cite{philip's}. --> also exists explainers specified for text classifiers~\cite{mine} for their specialty --> replace specialty with short phrase 
% Differentialale explanations for regularizing the behavior of the target models

% Fairness in machine learning has become a popular topic in the last decade...
% However, 

\section{Automatic misuse detector} \label{sec:mid}
% Before you begin to format your paper, first write and save the content as a 
% separate text file. Complete all content and organizational editing before 
% formatting. Please note sections \ref{AA}--\ref{SCM} below for more information on 
% proofreading, spelling and grammar.

% Keep your text and graphic files separate until after the text has been 
% formatted and styled. Do not number text heads---{\LaTeX} will do that 
% for you.
Before introducing the misuse detector, we specify the targeting bias to mitigate in this paper.
Imbalanced data distribution is a common problem while solving tasks with data-driven approaches.
The same problem also holds for hate speech detection.
For example, there are more toxic speeches against specific demographic groups on social media platforms because of social bias and stereotypes.
Hate speech detectors trained on these datasets could be biased while observing the high co-occurrence between the group identifiers and the hateful sentiment, which means they may use sensitive identifiers (e.g., ``Muslim'', ``Jewish'') as evidence for their prediction.
This kind of bias is highly undesired and is referred to as ``wrong reasons'' since it is an obvious misuse of features.
Here, we summarize the misuse into two cases:
\begin{itemize}
    % \item \textit{Wrong for wrong reasons}: a model gives \textbf{incorrect} predictions based on \textbf{irrelevant} observations.
    % \item \textit{Right for wrong reasons}: a model gives \textbf{correct} predictions based on \textbf{irrelevant} observations.
    \item \textit{\textbf{Wrong for wrong reasons}}: a model delivers \underline{wrong} decisions based on \underline{incorrect} reasoning of the observations.
    \item \textit{\textbf{Right for wrong reasons}}: a model delivers \underline{right} decisions based on \underline{incorrect} reasoning of the observations.
\end{itemize}
Regarding the first case, it is natural that problematic inference produces wrong predictions. 
% However, models could also output accurate classification for biased reasons as shown in Fig.~\ref{TOD1}.
However, models could also coincidentally output accurate classification for biased reasons.
Detecting misuses in the second case is extremely difficult as we cannot distinguish unreasonable right decisions from the justified ones without external resources.
% (e.g. human supervision or semantic meaning of words)
Therefore, we decided to concentrate on bias and misuse that result in wrong decisions as the first step towards automatic bias detection in text classification.

% The first case means that the model accidentally makes the right decisions based on wrong observations. 
% A typical example is shown in Fig.~\ref{}, that a speech is classified as hateful because of the group identifier rather than the hateful content.
% However, detecting misuses of this kind could be extremely difficult as we cannot distinguish unreasonable right decisions from those which are justified without external resources (e.g. human supervision or semantic meaning of words). 
% On the other hand, misuses in the latter case (for example Fig.~\ref{}) are more likely to be detected since potentially misused words can be filtered out by identifying those which contribute to the wrong decisions.
% Therefore, we decided to concentrate on the second case as the first step towards an automated debiasing framework for text classifiers.

In the following parts of this section, we introduce the proposed automatic misuse detector (\methodName).
It firstly filters out potentially biased words using an efficient proxy of feature importance (Section~\ref{sec:mid_base}, Section~\ref{sec:mid_proxy}) and then conducts a more detailed investigation using an explanation method to finalize the list of discriminated words (Section~\ref{sec:mid_detect}).

\subsection{Wrong for wrong reasons} \label{sec:mid_base}
% Given a hate speech detector $\mathcal{M}(\cdot)$ and a dataset $\mathcal{D}$ containing input texts $x\in\mathcal{D}$ with each text consisting of a list of tokens in sequential order $x=\{w_0, w_1, ..., w_n\}$, we define a wrong reason for wrong decisions in \eqref{eq:def_wrong} as a feature $w^*$ that consistently contributes to the incorrect prediction,
Given a hate speech detector $\mathcal{M}(\cdot)$ and a dataset $\mathcal{D}$ containing input texts $x\in\mathcal{D}$ with each text consisting of a list of tokens (words) in sequential order $x=\{w_0, w_1, ..., w_n\}$, we define a wrong reason for wrong decisions in \eqref{eq:def_wrong} as 
a feature $w^*$ whose average contribution $\bar{\phi}(w^*)$ to the misclassified instances is greater than a given threshold $\tau$,
\begin{equation} \label{eq:def_wrong}
    \bar{\phi}(w^*)_{\hat{y}\neq y}=\mathbb{E}_{\hat{y}\neq y}[\phi(w^*)]>\tau %\gg 0
\end{equation}
where $\phi(\cdot)$ denotes the function of deriving feature importance, $\hat{y}$ denotes the prediction on an instance containing the wrong reason $w^*$, and $y$ denotes the ground truth of the corresponding instance.
% It is shown in our experiments (Section~\ref{sec:exp_impacts}) that the misuse of features learned by a model could easily bring biases against certain groups into the classification.
% a potentially misused word $w^*$ as the word with the high feature importance score that contributes to a wrong decision. 
% Given a hate speech detector $\mathcal{M}(\cdot)$ and a dataset containing input texts $x\in\mathcal{D}$, where each text consist of a list of tokens in sequential order $x=\{w_0, w_1, ..., w_n\}$, we define a wrong reason for the wrong decision as a word (feature) $w^*$ with a significant contribution to the prediction of an incorrect class label.
% It is shown in our experiments (Section~\ref{sec:exp_impacts}) that the misuse of features learned by a model could easily bring biases against certain groups into the classification. 	

For the sake of simplicity, we here assume the given model $\mathcal{M}$ to be a binary classifier with the output value ranging from 0 to 1, where an output close to $1$ indicates the positive class ($\hat{y}=1$) and vice versa.
The contribution of a feature to the prediction can be measured using feature importance score (also known as feature attribution) which is the output of explanation methods. 
We choose a state-of-the-art explanation method named Sampling and Occlusion (SOC)~\cite{jin2019towards} for our debiasing framework, which defines feature importance as:
\begin{equation*}
    \phi(w^*)=\mathbb{E}_{\mathcal{X}_\delta}[\mathcal{M}(x)-\mathcal{M}(x\backslash w^*)]
\end{equation*}
where $\mathcal{X}_\delta$ denotes the context of the relevant input $x$.
The context consists of the neighboring instances derived from the input $x$ by randomly masking words with padding tokens.
The contribution of the feature $w^*$ is measured by the expected impact of excluding the target feature from all variants in the given context.
With the contribution defined, an intuitive solution for identifying wrong reasons would be finding features whose importance scores are highly correlated to the occurrence of misclassifications.
But this is unrealistic at runtime as computations for explaining all inputs are unaffordable.
To this end, we propose in the next subsection a loosened proxy of attribution for filtering out the suspicious features, which allows us to conduct the detection in real-time.

It has to be mentioned that other popular explanation methods (e.g., SHAP~\cite{lundberg2017unified}), which is differentiable, could also be adopted, but SOC is in favor for the efficiency reason. 
It can concentrate on the filtered out features rather than the whole input and thus helps reduce the computational complexity.

\subsection{False positive proportion as a proxy of feature attribution} \label{sec:mid_proxy}
It is more urgent to solve bias that leads to misclassification as hateful than the other way around.
Hence, we focus on the false positive instances\footnote{Instances ought to be negative but misclassified as positive.}. 
But the same theory also applies to the opposite case.
% For the false positive decision made on the input $x$, the feature $w^*$ which is mainly responsible for the prediction possesses the importance score $\phi(w^*)\gg0$.
For the false positive decision made on the input $x$, a feature which is mainly responsible for the prediction possesses the average importance score $\bar{\phi}(w^*)_{FP}>\tau$.
Although finding words with high attribution values which frequently appear in false positive instances could be an ideal way to determine the potentially misused features, computations carried by explanation methods are usually expensive \cite{rieger2020interpretations}. 
% It becomes unbearable if such computations have to be carried through the whole vocabulary since the size of vocabulary in up-to-date NLP models tends to be very high ($\sim 30k$~\cite{}).
It becomes unbearable if such computations have to be carried throughout the whole vocabulary since the size of vocabulary in up-to-date NLP models can easily go up to the order of tens of thousands~\cite{chen2019large}.
Therefore, we use the false positive proportion as an efficient proxy for the detection of the aforementioned wrong reasons supported by the Proposition \ref{proxy_prop}.
The false positive proportion of a feature $FPP_w$ is defined in \eqref{fpp_w}.
% In  \eqref{fpr_w}, the false positive proportion $FPP_w$ of a feature $w$  is defined as the proportion of false positive instances among all instances that contain $w$.
\begin{equation}\label{fpp_w}
    \begin{split}
        FPP_w=&~\frac{FP_w}{FP_w+FN_w+TP_w+TN_w}
        % =\frac{FP_w/|\mathcal{D}|}{(FP_w+TN_w)/|\mathcal{D}|}\\
        % =&~\frac{\textbf{Pr}(\hat{y}=1,y=0,w)}{\textbf{Pr}(w)}\\
        % =&~\textbf{Pr}(\hat{y}=1,y=0|w)
    \end{split}
\end{equation}
Note that the FPP is slightly different from the false positive rate by definition.
It is the ratio of false positive samples among all relevant samples rather than the relevant negatives.

\begin{prop} \label{proxy_prop}
% TBD: what is consistently
% TBD: clearer definition of the Proposition
% $\overhead{\phi}_{x\in FP}(w)\gg 0$
For a wrong reason $w^*$ where $\phi(w^*)>\tau$, its false positive proportion satisfies $FPP_{w^*}>\tau$.
\end{prop}
\begin{proof}
% Given a feature $w^*$ which is responsible for false positive predictions denoted as $FP$, we have:
Given a feature $w^*$ and the set of false positive instances $FP_{w^*}$ possessing the feature $w^*$, if $w^*$ is responsible for the misclassification, we have:
\begin{align*}
    \bar{\phi}(w^*)_{FP_{w^*}}=&~\mathbb{E}_{FP_{w^*}}[\phi(w^*)] \\
    =&~\mathbb{E}_{FP_{w^*}}[\mathbb{E}_{\mathcal{X}_\delta}[\mathcal{M}(x') - \mathcal{M}(x'\backslash w^*)]] \\
    =&~\mathbb{E}_{FP_{w^*}}[\mathcal{M}(x) - \mathcal{M}(x\backslash w^*)]\\
    =&~\mathbb{E}_{FP_{w^*}}[\mathcal{M}(x)] - 
    \mathbb{E}_{FP_{w^*}}[\mathcal{M}(x\backslash w^*)] > \tau \\
    \iff&~
    \mathbb{E}_{FP_{w^*}}[\mathcal{M}(x)] >
    \mathbb{E}_{FP_{w^*}}[\mathcal{M}(x\backslash w^*)] + \tau 
    % \\
    % \implies&~
    % \mathbb{E}_{FP_{w^*}}[\mathcal{M}(x)] > \tau
\end{align*}
where $x'$ indicates an instance from the context $\mathcal{X}_\delta$ and $x$ denotes a text that belongs to the set $FP_{w^*}$.
Here we assume that the average impact of simply masking out $w^*$ from all $x\in FP_{w^*}$ is equal to the mean of the feature importance, which is derived from the expected impact of excluding $w^*$ in the given context $\mathcal{X}_\delta$.
% with $x_\delta\in FP$ denoting contexts (which are generated by the explanation method for deriving feature attribution) related to the false positive decisions.
The greater expected model outcome implies that the given text containing the word $w^*$ is more likely to be classified as positive, and thus:
\begin{align*}
    \implies&~
    \textbf{Pr}(x\in FP_{w^*}|w^*)>\textbf{Pr}(x\in FP_{w^*}|\bar{w}^*) + \tau \\
    % \mathbb{E}_{x_\delta\in FP}[\mathcal{M}(x)] \geq \tau \\
    \iff&~
    \textbf{Pr}(x\in FP_{w^*}|w^*) > \tau \\
    \iff&~
    \frac{\textbf{Pr}(x\in FP_{w^*},w^*)}{\textbf{Pr}(w^*)}>\tau \\
    % \textbf{Pr}^{\mathcal{M}}(\hat{y}=1,y=0|w)\gg\textbf{Pr}^{\mathcal{M}}(\hat{y}=1,y=0)\geq 0 \\
    % \textbf{Pr}(\hat{y}=1,y=0|w)\gg\textbf{Pr}(\hat{y}=1,y=0|\overline{w})\geq 0 \\
    % \iff&~
    % \frac{\textbf{Pr}(\hat{y}=1,y=0,w)}{\textbf{Pr}(w)}\gg 0 \\
    \iff&~
    \frac{FP_{w^*}/|\mathcal{D}|}{(FP_{w^*}+FN_{w^*}+TP_{w^*}+TN_{w^*})/|\mathcal{D}|}>\tau \\
    \iff&~
    FPP_{w^*} > \tau
\end{align*}
where the symbol $\bar{w}^*$ indicates that the feature $w^*$ is masked out.
\end{proof}

For better efficiency of the proxy, we choose another threshold $\eta$ ($\eta\geq\tau$) for the false positive proportion $FPP_{w^*}$ as a complement of the neglected term $\textbf{Pr}(x\in FP|\bar{w}^*)$ on the right side of the inequation during the proof.
However, words with FPPs satisfying $FPP_w>\eta$ are not directly applicable for the debiasing purpose since the reverse of Proposition \ref{proxy_prop} does not hold as stated in Proposition \ref{lower_prop}.
% However, the reverse of Proposition \ref{proxy_prop} does not hold as stated in Proposition \ref{lower_prop}.
\begin{prop} \label{lower_prop}
For a word $\dot{w}$, the significance of its false positive proportion ($FPP_{\dot{w}}>\tau$) does not imply that it is a considerable reason for the wrong predictions.
% TBD: Significant false positive proportion of a word ($FPP_w\gg 0$) does not imply that the word is a considerable wrong reason
\end{prop}
\begin{proof}
Given a word $w$ as a wrong reason, if there exists another word $\dot{w}$ which has unconditional high co-occurrence with $w$ (e.g., for semantic/syntactic reasons), it means:
% \begin{gather}
\begin{align}
    &\textbf{Pr}(w)=\textbf{Pr}(w, \dot{w})=\textbf{Pr}(\dot{w}) \label{eq_pr_overall}\\
    &\textbf{Pr}(w|x\in FP)=\textbf{Pr}(w,\dot{w}|x\in FP)=\textbf{Pr}(\dot{w}|x\in FP) \label{eq_pr_cond}
\end{align}
% \end{gather}
The statement about the significant false positive proportion holds for $\dot{w}$ without any constraints on its importance score $\phi(\dot{w})$.
\begin{align*}
    % \textbf{Pr}(\dot{w}\in FP| w) =&~\frac{\textbf{Pr}(x\in FP, \dot{w})}{\textbf{Pr}(\dot{w})}\\
    \textbf{Pr}(x\in FP| \dot{w})=&~\frac{\textbf{Pr}(\dot{w}|x\in FP)}{\textbf{Pr}(\dot{w})}\cdot\textbf{Pr}(x\in FP)\\
    \overset{\eqref{eq_pr_overall}}{=}&~\frac{\textbf{Pr}(\dot{w}|x\in FP)}{\textbf{Pr}(w)}\cdot\textbf{Pr}(x\in FP)\\
    \overset{\eqref{eq_pr_cond}}{=}&~\frac{\textbf{Pr}(w|x\in FP)}{\textbf{Pr}(w)}\cdot\textbf{Pr}(x\in FP)\\
    =&~\textbf{Pr}(x\in FP| w) > \tau\\
    \iff&~FPP_{\dot{w}}>\tau
    % \eqref{eq_pr_cond}
    % \overset{Bayes}{=}
\end{align*}
\end{proof}
Proposition \ref{lower_prop} indicates that the false positive proportion is a loosened proxy of the feature attribution, which requires an in-depth analysis of the identified words to condense the word list before it serves for the debiasing task.

% TBD: a loosen restriction for the co-occurrence could be mentioned here
Computation of the FPP can be done in linear time (dependent on the size of the dataset) as model predictions are available during training time. 
Hence, its computation is much more efficient than exhaustively performing the chosen explanation method for feature attribution.
% TODO: more elegant estimation of the hyperparameter \eta
According to Proposition \ref{proxy_prop}, we apply FPP as a proxy to filter out the potentially biased words and then utilize the SOC explainer to exclude features with limited contributions from the final list following Proposition \ref{lower_prop}.
% Proposition \ref{proxy_prop} allows us to apply FPP as a proxy for filtering out the potentially biased words and then utilize the SOC explainer to exclude features with limited contributions from the final list following Proposition \ref{lower_prop}.
Note that obtaining feature importance is practical after the filtering as the number of words under investigation is about $0.2\%$ of the vocabulary size.

\subsection{Misuse detection} \label{sec:mid_detect}
For the trade-off between accuracy and fairness, we intend to debias a model with minimal external restrictions applied. 
A model is updating itself for the given task during training. 
Misconduct of the model at the current step could be self-corrected after finishing further training steps.
% Accordingly, we construct a sliding window with size $l$ in the misuse (potential bias) detector,
% Only features with at least $k$ FPP records in the sliding window exceeding the user-defined threshold $\eta$, namely $FPP_w>\eta$, will be added in the candidate list.
Accordingly, the misuse detector in the debiasing framework is designed to examine FPPs of features periodically during the training phase instead of completing the detection at once.
\methodName~integrates a sliding window with size $l$ recording the most recent feature statistics. 
Only features with at least $k$ records in the sliding window, whose FPP exceeds the user-defined threshold $\eta$, namely $FPP_w>\eta$, will be added to the candidate list.
Fig.~\ref{fig:absorb} demonstrates an example of the candidate list absorbing and rejecting the features according to their FPPs.
\begin{figure}[htbp]
    \centering
    \includegraphics[width=0.35\textwidth]{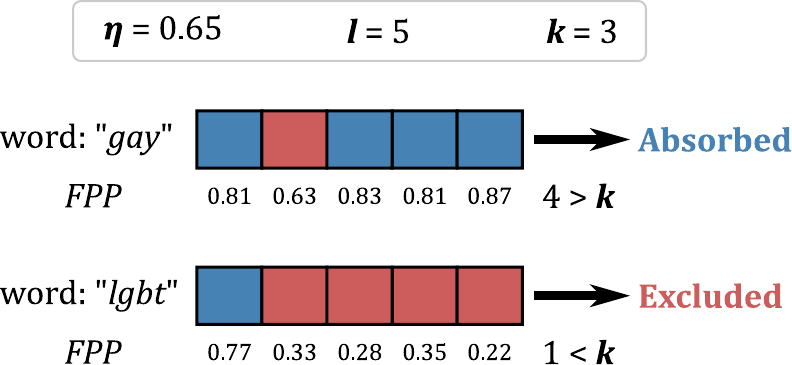}
    % \includesvg[width=0.35\textwidth]{pic/absorb_example.svg}
    \caption{Absorbing and excluding potentially biased features through \textit{FPP}}
    \label{fig:absorb}
\end{figure}
% TBD: GPT-3

Words filtered out by FPP are candidates of debiasing targets.
And because of Proposition \ref{lower_prop}, investigation on their feature importance scores is mandatory.
Only candidates with an importance score assigned by SOC greater than the given threshold $\tau$ will be inserted into the debiasing list $W$.
The finalized debiasing list $W$ will contribute to the correction of the model.

\section{Staged training} \label{sec:staged}
Based on \methodName~proposed in previous subsection, we introduce a staged training process for automatic debiasing in this section.
The training process consists of three stages: 
i) \textit{vanilla}, 
ii) \textit{correction}, and 
iii) \textit{stabilization}, which are listed following their orders in the training process.
The training process enters the next stage when $n$ training iterations are completed.

At the \textbf{vanilla} stage, no restrictions are applied to the training target.
The misuse detector is recording statistics of suspicious features periodically (as entries in the sliding window) in parallel for the construction of the debiasing list $W$.
Words inserted into $W$ remain in the list even if their importance scores drop below the threshold $\tau$.
The debiasing list $W$ is thus continually growing until the vanilla stage ends.
The \textbf{correction} stage follows the end of the vanilla. 
Misconduct of the model is corrected at this stage with restrictions being applied to the target model based on the debiasing list $W$.
To implement the restrictions, the objective function is updated with a regularized explanation term added to the classification objective $\mathcal{L}$ as follows:
\begin{equation*}
    \mathcal{L}_E=\mathcal{L} + \lambda\sum_{w\in W}|\hat{y}-y|\cdot\phi(w)^2
\end{equation*}
The second term is the regularization term with the strength controlled by a hyperparameter $\lambda$.
Since our focus is solving \textit{wrong for wrong reasons}, penalties are laid solely on misclassified instances containing features listed in $W$.
The debiasing list remains static at this stage regardless of the possible changes of either feature attribution or FPP.
After the correction, the debiasing list is resolved. 
And finally, as the \textbf{stabilization} stage, the model is trained without restrictions for another $n$ steps.

The first two stages are critical in the training process.
Unlike the other debiasing methods for text classifiers, which define the debiasing list manually ahead of the training, we first train the target model for $n$ steps and actively search for the discriminated words.
It encourages the debiasing framework to be more responsive to the actual model behaviors, which could be affected by trivial changes (e.g., varying random seeds~\cite{zhuang2022randomness}) in the training settings.
The variable nature of the training process is another reason why a pre-defined debiasing list should be less favored than maintaining a list at runtime.
We apply the extracted knowledge to enforce the correction at the second stage.
Though it is not mandatory, the stabilization step is arranged to enable the model to optimize itself but at the risk of reappearing recovered bias.

\section{Experiments} \label{sec:exp}
We designed comprehensive experiments to study the performance of the proposed debiasing framework. 
Firstly, we analyze the outputs of \methodName~in a qualitative manner (Section~\ref{sec:exp_mid}).
Secondly, we compare the model which is trained using the staged debiasing framework to its competitors (Section~\ref{sec:exp_debias}). % under various evaluation metrics 
Last but not least, an investigation to uncover the indirect impacts of debiasing is conducted.
And we discuss the overall coverage of \methodName~on the manual list considering both direct and indirect debiasing impacts (Section~\ref{sec:exp_impacts}).
The concrete experimental settings are discussed in the coming Section~\ref{sec:exp_setup}.
\subsection{Experiment Details} \label{sec:exp_setup}
% Dataset
% Text classifier
% parameter settings
% Competitors
\textit{\textbf{Dataset}}:
We evaluate our approach on the "Gab Hate Corpus"~\cite{kennedy2018gab} dataset.
The dataset contains $27,655$ speeches sampled from a social network named "Gab" which is populated by "Alt-right"~\cite{anthony2016inside} with a high rate of hate speech. 
% The dataset contains $27,655$ speeches sampled from a social network named "Gab" which is populated by "Alt-right"~\cite{anthony2016inside} with a high rate of hate speech~\cite{TOD1}. 
There are $22,036$ instances in the training set with $1941$ of them being labeled as hateful ($y=1$) by human annotators and the remaining as non-hateful ($y=0$).
The validation set has a size of $2,755$ with $245$ instances as hateful.
The test set has the same size as the validation set and includes $264$ positive instances.
% training: 22035, 1941, 20095
% evaluate: 2755, 245, 2510
% test: 2755, 264, 2491

\textit{\textbf{Text classifier}}: 
The debiasing target is a text classifier solving the hate speech detection task given by the aforementioned dataset. 
We select BERT~\cite{kenton2019bert}, a state-of-the-art language model based on the attention mechanism~\cite{vaswani2017attention}, for our experiments. 
Although BERT has been utilized for numerous downstream tasks, recent researches demonstrate various biases (e.g. gender~\cite{bhardwaj2021investigating}, ethnic~\cite{ahn2021mitigating} bias) it could induce.
We fine-tune the pre-trained version that is publicly available in the transformers library~\cite{wolf2020transformers}.
% TBD: more detailed training settings should be attached in the appendix

\textit{\textbf{Hyperparameters}}:
For the misuse detector \methodName, the two thresholds $\eta$ (for FPP) and $\tau$ (for feature importance score) give the definition of misused terms. 
A strict definition with high thresholds will lead to the risk of overlooking biases; meanwhile, a detector with low thresholds may involve irrelevant tokens and thus cause damage to the model's performance.
We here determine $\eta$ and $\tau$ via grid search and set them to $0.65$ and $0.45$, respectively.
The misuse detector records FPPs of suspicious features every $50$ iteration with the sliding window size $l$ equals $10$ and the requirement of the minimal count $k$ equals $7$.
As for the debiasing framework, the strength of the explanation regularization is $0.1$ following the optimal parameter setting in \cite{kennedy2020contextualizing}.
The training iteration is set to $1,000$ equally for each stage, which indicates $3,000$ training iterations in total.
% We empirically choose the hyperparameters for \methodName~and the debiasing framework.
% The two important thresholds $\eta$ (for FPP) and $\tau$ (for feature importance score) are set to $0.45$ and $0.65$ respectively.
% The strength of the explanation regularization is $0.1$. 
% And the misuse detector records FPPs of suspicious features every $50$ iterations with the sliding window size $l$ equals $10$ and the requirement of the minimal count $k$ equals $7$.
% The training iteration is set to $1,000$ equally for each stage, which indicates $3,000$ training iterations in total.

\textit{\textbf{Competitors}}: 
We compare the model corrected by the proposed framework to the two models with the same structure but trained under various circumstances.
One is trained under the vanilla setting with no restrictions.
The other is trained with the debiasing method proposed in~\cite{kennedy2020contextualizing}, which is considered as a baseline.
The baseline employs a \textit{manually} defined static debiasing list and the strength of the regularization is the same as our method (i.e., $\lambda=0.1$).
The static debiasing restrictions take effect during the whole training process, while ours educates the model for fairness only during the middle stage (correction).
For both the vanilla setting and the baseline, the models are trained for $3,000$ iterations, which is identical to the total amount of training iterations in the staged correction.

\subsection{Evaluating misuse detector} \label{sec:exp_mid}
% potentially biased words
%   1. sensitive group identifiers --> biased
%   2. sentimental words --> over-sensitive
%   3. uncovered words in the manual list --> could be the case: right for wrong reasons, which are not the focus of the paper
We first evaluate the misuse detector as it is the base of the proposed debiasing framework.
The list of potentially biased words given by \methodName~during the first stage is presented in Table~\ref{tab:word_table} along with the manually selected ones used in the baseline.
The words shared by the two lists are marked as bold and summarized in the first row.
And for words contained in the manual list, the box plots of the biased feature attribution and the visualization of corresponding FPPs can be found in Fig.~\ref{fig:attr_manual}.
\begin{table}[tbp]
    \caption{Lists of words for debiasing}
    \centering
    \begin{tabular}{cc}
        \hline \\[-1.7mm]
        % \\
        \makecell[c]{Shared by\\both lists} & \makecell[l]{\textbf{muslim} \textbf{muslims} \textbf{islam} \textbf{islamic} \textbf{jew} \textbf{jews} \textbf{jewish} \textbf{gay}\\
        \textbf{white} \textbf{whites} \textbf{black} \textbf{blacks}}\\
        \\[-1.7mm]
        \hline
        \\[-1.7mm]
        \makecell[c]{Only in\\\methodName} & \makecell[l]{
        immigrant immigrants left liberal democrats communist\\ 
        african racist nazi leftist corrupt migrants liberals traitor \\
        homo rap slave terrorist fucking gender} \\
        \\[-1.7mm]
        \hline
        \\[-1.7mm]
        \makecell[c]{Only in\\Manual list} & \makecell[l]{ woman women democrat allah  lesbian transgender brown \\
        race mexican religion homosexual homosexuality africans}\\
        \\[-1.7mm]
        \hline
    \end{tabular}
    \label{tab:word_table}
\end{table}
\begin{figure}[htbp]
    \centering
    \includegraphics[width=0.45\textwidth]{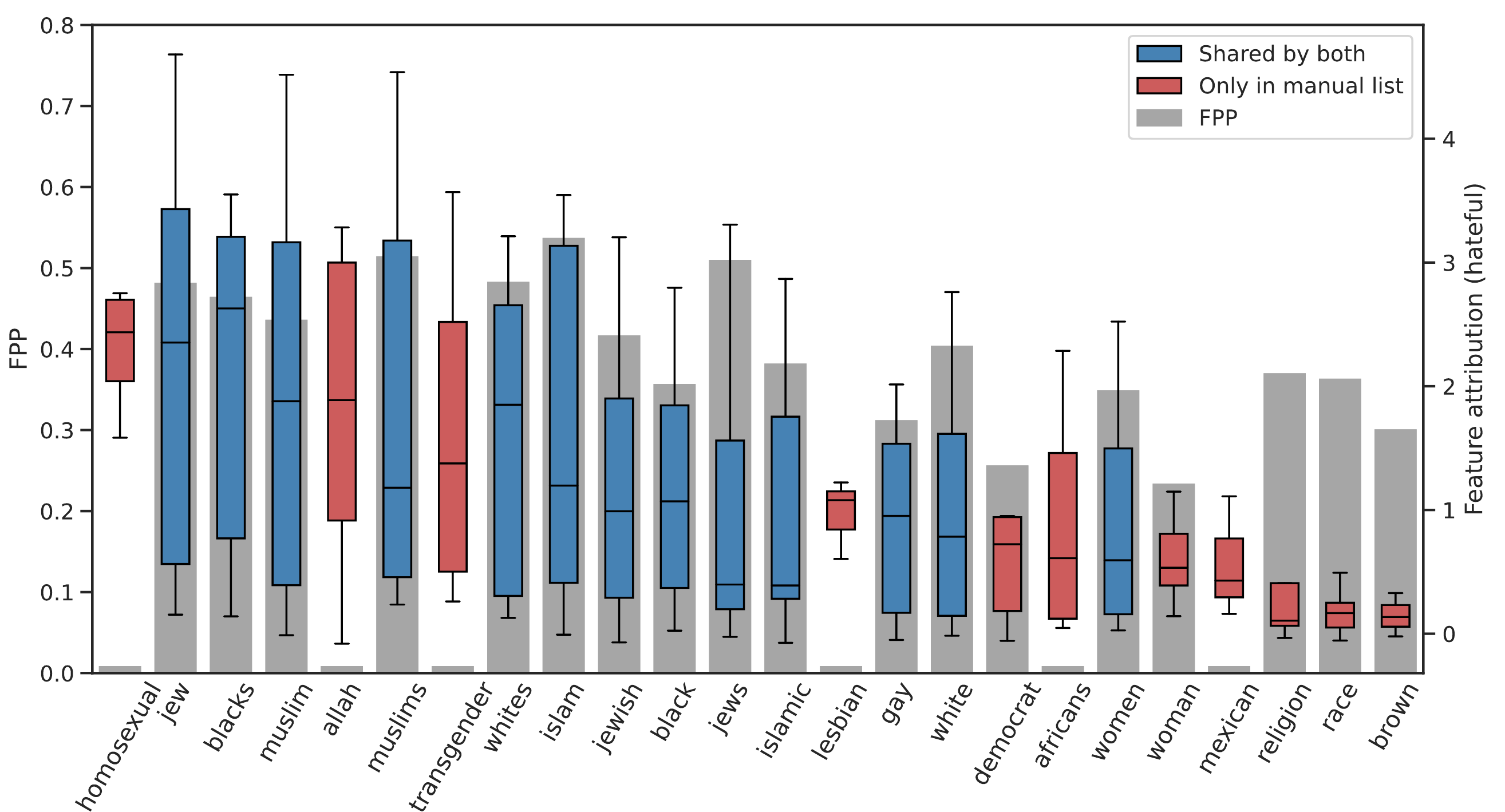}
    % \includesvg[width=0.45\textwidth]{pic/attr_manual.svg}
    \caption{Feature attribution of the manually selected words}
    \label{fig:attr_manual}
\end{figure}
% . The manually constructed list can be found in Table~\ref{}.
% We noticed that the automatically extracted list successfully identified a certain amount of words included in the other list. 
% For example, group identifiers (such as "-" and "-"), which are inappropriate to be considered as observations for hateful speech detection, are covered by both lists.
The automatically constructed debiasing list covered 12 out of 25 words given in the manual list, which is considered sensitive and potentially discriminated.
For example, sensitive group identifiers (such as ``\textit{muslim}'', ``\textit{black}'', and ``\textit{gay}'') are covered by both lists.
In addition, \methodName~successfully identifies words referring to demographic information, like ``\textit{immigrant}'' and ``\textit{liberal}'', which are overlooked by the human annotators.

At the same time, we are not surprised that potentially hateful words such as "\textit{nazi}" and "\textit{fucking}" are also listed. 
Unreasonably referring to an individual or a group of people as ``\textit{nazi}'' is truly hateful and should be filtered out before it gets published in public environments. 
But the model may ban proper expressions or discussions (e.g., on histories) if it becomes over-sensitive to the potential evidence for hate sentiments.
% TBD: consider to add examples
Although the main focus of the work is to mitigate biases carried by hate speech classifiers, over-sensitivity towards potential evidence may also be the wrong reason that is responsible for wrong decisions.
It is also a reason why we named the proposed method as misuse detector rather than bias detector.

Furthermore, the box plots of feature attribution in Fig.~\ref{fig:attr_manual} show that the misuse detector excludes less discriminated words like ``\textit{race}'' and ``\textit{brown}'', which exist in the manual list,  from the automatic one. 
Despite the truth that human annotators consider these words discriminated, they play a neutral role (told by the relatively low feature importance scores shown in the box plots) during the decision-making process. 
In principle, we could include all neutral words in the debiasing list for models' fairness. 
% TBD: try to improve this
But since the experiment in Section~\ref{sec:exp_impacts} demonstrates that attribution suppression on single words would affect their semantic neighbors, we argue that the debiasing list should be concise for model capacities by excluding less biased words.

Apart from the merits, the detector overlooks some manually selected words that are misused, such as ``\textit{homosexual}'' and ``\textit{transgender}''.
Illustrated by their considerably lower FPPs, the main reason for the incompleteness is that these wrong reasons lead to \textbf{right} predictions.
While concentrating on correcting wrong decisions, we selectively neglect the case ``\textit{right for wrong reasons}'' as discussed in Section~\ref{sec:mid}.
Nevertheless, we demonstrate in Section~\ref{sec:exp_impacts} that the debiasing framework lays approving impacts on these words (feature importance scores dropped over $50\%$) even though they are not directly affected.

\subsection{Evaluating debiasing framework} \label{sec:exp_debias}
% Training loss over iterations
% Final performance over stages, comparing to competitors
% Debiasing results
To reveal the impact of the debiasing framework on training a model, we present in Fig.~\ref{fig:loss} the training losses under different circumstances.
\begin{figure*}[tbp]
    \centering
    \includegraphics[width=1.\textwidth]{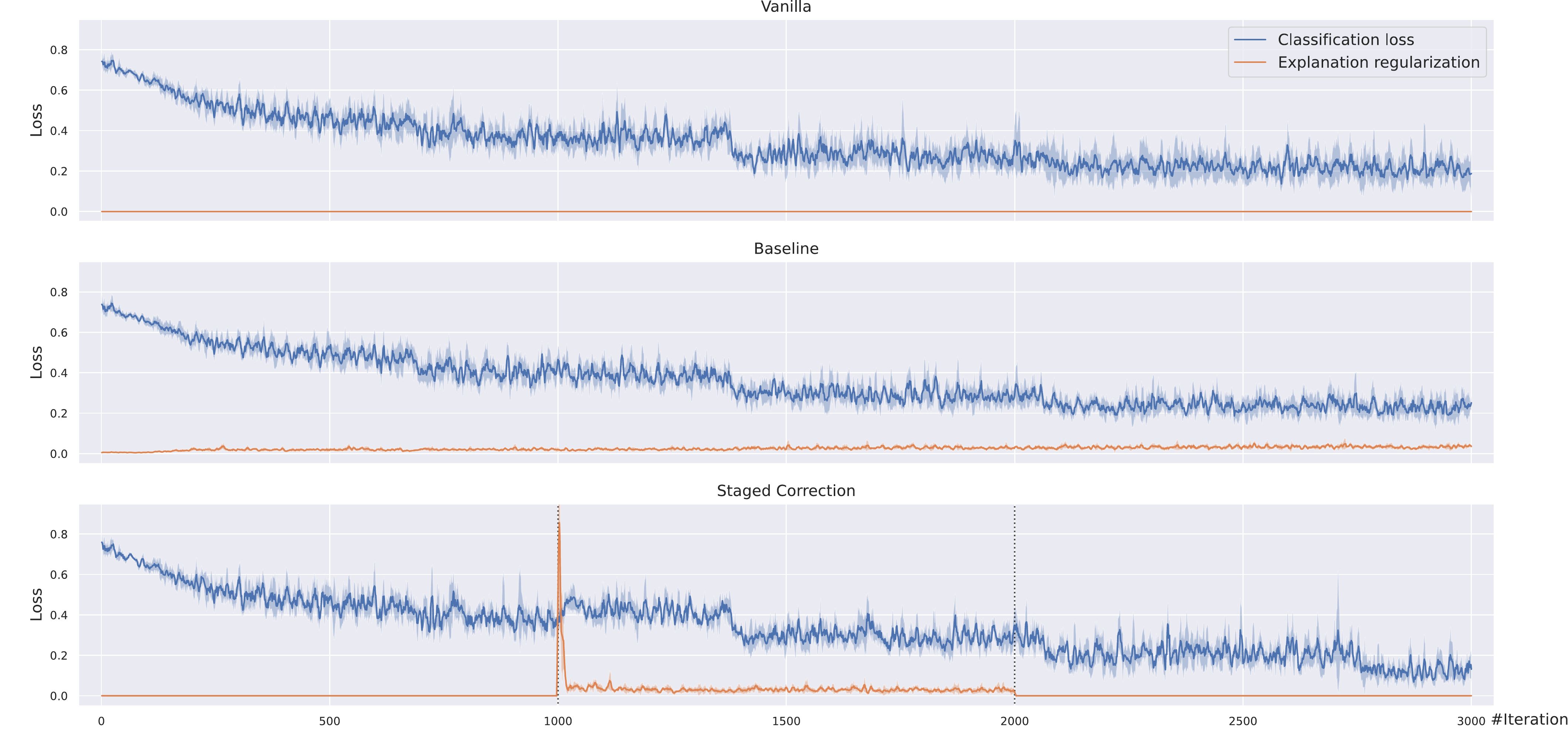}
    % \includesvg[width=1.\textwidth]{pic/loss.svg}
    \caption{Training loss over iteration with/-out debiasing}
    \label{fig:loss}
\end{figure*}
Having the lines in blue indicate the pure classification loss (cross-entropy), the orange lines refer to the penalties regularized by explanations, and the sum of the two values is the final back-propagated loss.
The model is trained using the same hyperparameters with the only difference in the random seed for 10 times under each setting, and the reported losses are the averaged values of the 10 versions.
For the \textit{vanilla} setting, the penalty introduced by the regularization term is identical to $0$ as no restrictions were applied.
And a decreasing trend can be observed for the loss until around the $1,500_{th}$ iteration.
The \textit{baseline} adopted a pre-defined list for debiasing, which remained in force during the whole process. 
As a consequence of the applied restrictions, the regularization term stays close to $0$, which indicates the extremely low attribution of the debiasing targets.
Finally, the loss fluctuates at a level close to the same value under the vanilla setting  with a similar progressing period.

As for the \textit{staged correction} method, we separate the training stages by the vertical dotted line at the $1,000_{th}$ and $2,000_{th}$ iterations.
The change of the classification loss is similar to the other two settings at the first stage. 
The restrictions for debiasing come into effect when the first stage ended.
Unlike what is shown in the baseline, the penalty here starts with a very high value at the beginning of the correction stage as certain misbehaviors of the model have been detected by \methodName~during the first $1,000$ iterations.
And the model manages to correct the misuse instantly (less than $100$ iterations) following the guidance of the explanations.
At the same time, since the constraints limit the model reasoning its predictions, we observe a notable increase in terms of classification loss.
But the model adapts to the changes quickly and starts improving its classification performance again.
The restrictions are resolved when the last stage begins, which leads to a further decrease in the classification loss.
And surprisingly, the model trained with the staged correction achieved the lowest classification loss among three.
A possible cause of the observation is that explanations as guidance for model attribution can improve its performance \cite{zhang2021explain}.
Although we do not bring the full expected feature attribution, excluding the irrelevant/wrong features may have a similar effect by reducing the search space.

For a more comprehensive study on the impact of debiasing approaches on the classification task, we also report the performance of the three versions of the same model in Table~\ref{tab:accuracy}.
And again, the model is trained $10$ times with various random seeds for each setting, and the averaged performance on the test set followed by the standard deviation is presented.
For the staged correction, we store the best performing model on the validation set for each stage separately and then evaluate them with the test set.
% \begin{table*}[tbp]
\begin{table}[tbp]
\caption{Classification performance on the test set}
    \centering
    \begin{tabular}{c|cccc}
        \hline \\ [-1.7mm]
        Setting & Acc. & F1 & Precision & Recall \\[1mm]
        \hline
        \\ [-1.7mm]
        Vanilla & 87.96\,±\,1.22 & \textbf{48.59\,±\,1.75} & 41.47\,±\,3.30 & 59.17\,±\,3.13 \\[1mm]
        % \hline
        % \\ [-1.7mm]
        Baseline & 86.85\,±\,1.29 & 44.88\,±\,1.26 & 37.85\,±\,2.60 & 55.83\,±\,4.77 
        \\[1mm]
        \hline
        \\ [-1.7mm]
        \methodName-Van. & 86.31\,±\,0.88 & 48.45\,±\,0.98 & 38.02\,±\,1.68 & \textbf{67.10\,±\,3.15}
        \\[1mm]
        % \hline
        % \\ [-1.7mm]
        \methodName-Cor. & 88.93\,±\,0.60 & 47.18\,±\,1.51 & 43.71\,±\,2.29 & 51.73\,±\,4.75
        \\[1mm]
        % \hline
        % \\ [-1.7mm]
        \methodName-Sta. & \textbf{89.06\,±\,1.42} & 48.02\,±\,1.52 & \textbf{45.08\,±\,5.23} & 52.53\,±\,5.36
        \\ [1mm]
        \hline
        % \\[-1.7mm]
        \multicolumn{5}{p{8.4cm}}{Accuracy, F1, Precision, Recall (\%) are reported on the test set. Van., Cor., and Sta. are the abbreviations of the vanilla, correction, and stabilization stages respectively.} \\
    \end{tabular}
    \label{tab:accuracy}
% \end{table*}
\end{table}
Models from the \methodName-Vanilla stage have the lowest averaged accuracy on the test set as it is trained only for $1,000$ iterations.
Although they achieve the highest recall in detecting hateful speeches, the obviously low precision indicates that many false positive decisions have been made, which agrees with the high FPPs visualized in Fig.~\ref{fig:attr_manual}.
During the correction stage, the biases of the training targets are mitigated, which results in the second highest precision.
As a side effect, constraints on the decision making process limit model performance in terms of recall. It also results in an average F1 score that is roughly $1.4\%$ lower than the highest obtained in the vanilla setting.
Consistent with the training loss, the final models (at the stabilization stage) delivered by the staged correction reach the highest accuracy among the competitors.
In fact, the performance is improved for all listed metrics in comparison with the results of the correction stage.
And even though the restrictions have been resolved, the mean precision increases rather than decreases.
On the contrary, the baseline, which also debiases the training target by suppressing unwanted attribution, attains a fairly poor performance with all the figures being lower than the vanilla ones.
A possible explanation for this observation is that the manually defined debugging list applied during the whole process limits the convergence of the training targets towards the global optimal.

In Fig.~\ref{fig:attr_compare}, we show the changes of averaged feature attribution (\textit{FA}) and averaged false positive proportion \textit{FPP} over iterations for the words extracted by \methodName.
For a better comparison between our debiasing framework and the baseline, we visualize the same statistics for the baseline.
We want to point out that the words behind the figures are not identical for the two methods as listed in Table~\ref{tab:word_table}.
% \begin{figure*}[tbp]
%     \centering
%     \includegraphics[width=1.\textwidth]{pic/attr_compare.png}
%     \caption{Changes of feature attribution (FA) and false positive proportion (FPP) over iterations in the staged correction and the baseline}
%     \label{fig:attr_compare}
% \end{figure*}
% It is truly inappropriate to use them as evidence for hateful speech detection.
Records of different training settings are distinguished by color. 
The solid and dotted lines indicate the averaged FA and FPP of the debiasing targets respectively.
A primary observation is the correlation between FA and FPP, especially for the synchronized drop at the beginning of the correction stage and the fluctuation with similar trends since the $1,500_{th}$ iteration.
The same observation holds for the plots of the baseline, whose debiasing list has an overlap with the \textit{wrong reasons} detected by \methodName.
But finding the correlation in the baseline requires a closer look as the plots are rather flattened.
The observed correlation supports our propositions and the usage of \textit{FPP} as a proxy of feature attribution for identifying wrong reasons for wrong decisions.
% As for the vanilla setting, we observe that both values kept increasing and stuck to a fairly high level, which implies a strong bias held by the given model.

The baseline maintained the lower feature attribution for the selected words during the whole process.
Suppression of the incorrect reasoning contributes to the low FPP ($\approx0.1$), which is only one-third of the averaged FPP ($\approx0.3$) of the same words in a biased model as shown in Fig.~\ref{fig:attr_manual}.

As for the proposed method, the records for the staged correction do not start at the first iteration, because the misuse detector requires time to confirm that the model consistently discriminates against certain features.
Corresponding to the massive regularization penalty at the beginning of the correction stage, the averaged feature attribution of the automatically detected words is relatively high before the correction ($\#iter\leq1,000$) in the proposed debiasing framework.
And a sharp decrease of the FA at the start of the correction stage matches the drop of the regularization term reported in the training loss.
The feature attribution keeps falling during the correction and becomes competitive at the end of the second stage to the same value in the baseline.
Simultaneously, the FPP is experiencing a decrease.
After entering the stabilization stage, the FA of the biased words starts recovering because of the release of the restrictions.
In spite of the slight increase in the FA, the decline of FPP continues.
Starting from an extremely high point of the FPP, our debiasing framework with the staged correction outperforms the baseline in the middle of the correction stage and ends up with the FPP approaching $0$, even though it possesses a larger debiasing list.
The loosened restrictions in the staged correction allow the model better explore the solution space and are considered to be the main cause of the better performance compared to the baseline.
% But the value started recovering after the release of the restrictions.
% On the contrary, the FPP generally follows a downward trend at the last two stages regardless of the existence of the restrictions.
% Because of the incomplete coverage on the given list, the staged correction holds a higher \textit{FIS} than the baseline does.
% However, as demonstrated by the FPP, the most strict restrictions in the baseline do not lead to the fairest performance.
% The loosened restrictions in the staged correction allowed the model better explore the solution space and attained the lowest FPP when the training ended. 
% The observation can lead to an interesting discussion, which will be included in our future work, namely: how do we deal with the trade-off between \textit{FIS} and \textit{FPP} when both refer to fairness?

\subsection{Indirect impacts of correction} \label{sec:exp_impacts}
% neighboring features of the target features could also be affected
% result in a better coverage in the manual list (direct + indirect)
% minimal size of the debiasing list to avoid unexpected impact
\begin{figure*}[tbp]
    \centering
    \begin{subfigure}[b]{\textwidth}
        \centering
        \includegraphics[width=1.\textwidth]{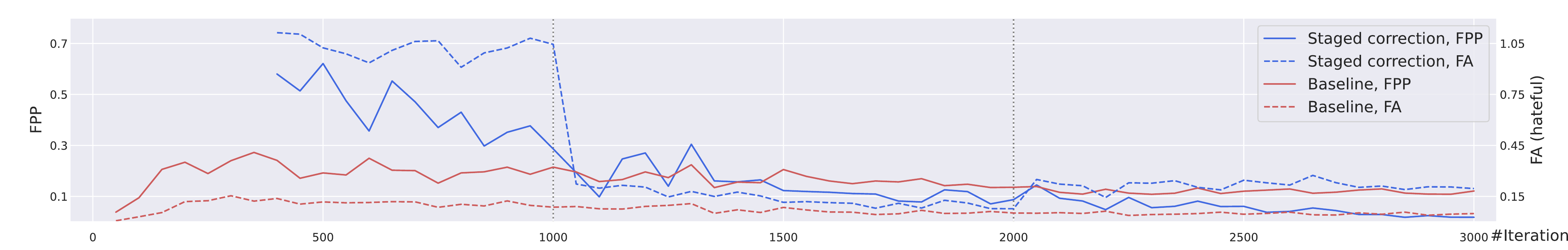}
        % \includesvg[width=1.\textwidth]{pic/attr_compare.svg}
        \vspace*{-6mm}
        \caption{Comparison between the staged correction and the baseline by FA and FPP}
        \label{fig:attr_compare}
    \end{subfigure}
    \begin{subfigure}[b]{\textwidth}
        \centering
        \vspace*{2mm}
        \includegraphics[width=1.\textwidth]{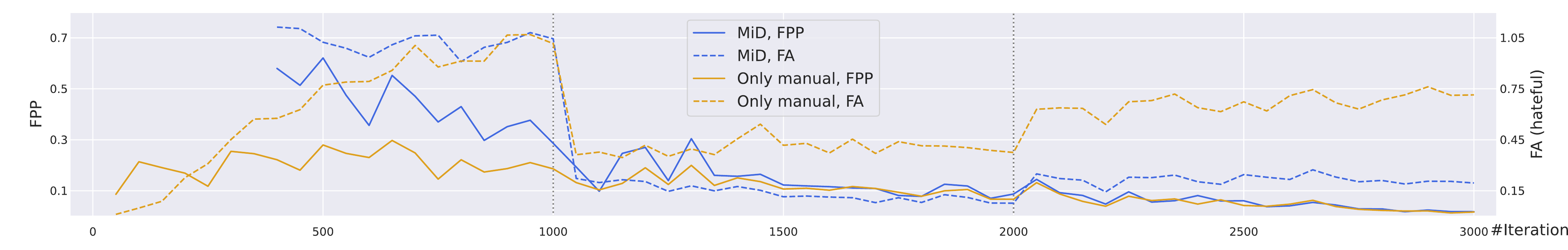}
        % \includesvg[width=1.\textwidth]{pic/attr_indirect.svg}
        \vspace*{-6mm}
        \caption{Changes of FA and FPP for words included (``MiD'') and excluded (``Only manual'') in the staged correction}
        % Changes of feature attribution (FA) and false positive proportion (FPP) for words included and excluded in the staged correction
        \label{fig:attr_indirect}
    \end{subfigure}
    \caption{Changes of feature attribution (FA) and false positive proportion (FPP) over iterations}
\end{figure*}
Despite the absence of some manually selected words in the automatically extracted word list, an impressive consistency between the changes of the feature attribution highlights itself in Fig.~\ref{fig:attr_indirect}.
% \begin{figure*}[tbp]
%     \centering
%     \includegraphics[width=1.\textwidth]{pic/attr_indirect.png}
%     \caption{Changes of feature attribution (FA) and false positive proportion (FPP) for words included and excluded in the staged correction}
%     \label{fig:attr_indirect}
% \end{figure*}
The figure demonstrates the FPP and FA changes for two groups of words in the staged debiasing framework: i) words included in the staged correction process; ii) words omitted by the detector but contained in the manual list.
While the debiasing process suppresses the feature attribution of the debiasing targets, the words that are not explicitly included also get affected.
Since the experiment is carried on the BERT model, which represents similar tokens with similar embeddings, we believe that correction on single words also brings the effects to their semantic neighbors as they are implicitly connected through semantic meanings.

Motivated by the idea, we conduct further experiments to verify the existence of the indirect impacts.
The size of the debiasing list extracted by \methodName~with the standard parameters is fairly large.
To prevent the possible cross effects among multiple features from introducing additional complexity to the analysis, we reduce the number of debiasing targets by raising the two thresholds in \methodName~to tighten the definition of the wrong reason.
By applying the change, the shortened debiasing list (in Table~\ref{tab:shorten_table}) contains $6$ words.
\begin{table}[tbp]
    \caption{Shortened debiasing list}
    \centering
    \begin{tabular}{c}
        \hline \\[-1.7mm]
        % \\
        \makecell[l]{\textbf{muslims} \textbf{islam} \textbf{jews} communist racist homo }\\
        \\[-1.7mm]
        \hline
    \end{tabular}
    \label{tab:shorten_table}
\end{table}
In addition to the automatically detected words, from the manual list, we select three words with high cosine similarities to the multiple terms in the previous list, namely ``\textit{muslim}''\footnote{This word is different from the word ``muslims'' as they are represented by different tokens in the classification task, the same holds for other terms.}, ``\textit{islamic}'', and ``\textit{jewish}''.
We also choose the two words ``\textit{homosexuality}'' and ``\textit{black}'' with each having medium similarity to one term and one word ``\textit{race}'', whose embedding is dissimilar to the others.
The correlation matrix of the changes in feature attribution and the cosine similarities among word embeddings are presented in Fig.~\ref{fig:indirect}.
In addition to the listed words 
\begin{figure}[htbp]
    \centering
    \begin{subfigure}[b]{0.24\textwidth}
        \centering
        \includegraphics[width=\textwidth]{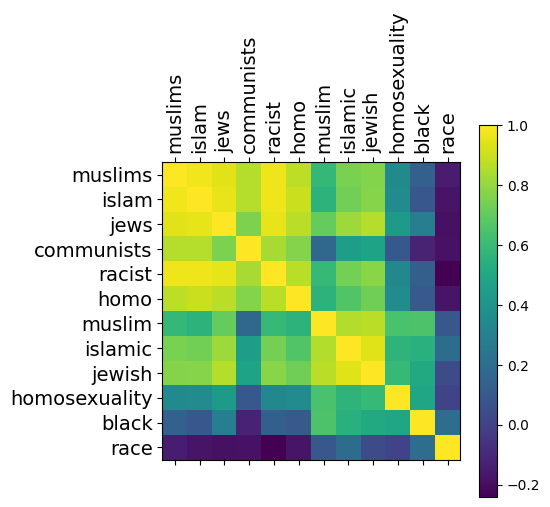}
        \caption{Correlation matrix of FA}
        \label{pic:corr}
    \end{subfigure}
    \hfill
    \begin{subfigure}[b]{0.24\textwidth}
        \centering
        \includegraphics[width=\textwidth]{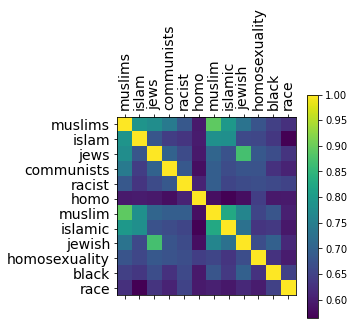}
        \caption{Cosine similarity}
        \label{pic:simlarity}
    \end{subfigure}
    \caption{Indirect impacts analysis}
    \label{fig:indirect}
\end{figure}
The bright square at the upper left corner of the correlation matrix in Fig.~\ref{pic:corr} is directly caused by the correction, all the $6$ words relating to the square are actively suppressed by the debiasing framework.
And the three words with high similarities to at least two tokens show a strong correlation in terms of feature attribution. 
For the remaining words, the lower correlations agree with the fact that their embeddings are less similar to the targeting ones.
The finding further confirms that suppressing feature attribution of words would indirectly affect their semantic neighbors in a similar way.
The more detailed analysis further reveals that overlaying the indirect effects (e.g., indirect effects originating from ``muslims'' and ``islam'' on ``islamic'') could strengthen the correlation.
Considering the semantically similar words shared by both lists in Table~\ref{tab:word_table}, we, therefore, argue that the proposed misuse detector has better coverage (direct and indirect) on the manual list than it appears to be.
% -------------------------------------
% TBD: show the direct coverage and the indirect coverage

% \section{Conclusion and future work} \label{sec:concl}
\section{Conclusion} \label{sec:concl}
In this work, we proposed \methodName, a fully automatic misuse detector with the main purpose of uncovering biases learned by a text classifier during the training phase.
With the deployment of the detector, the training object is debiased through a staged correction process.
Our experiments show the outstanding coverage of the automatically extracted list on the discriminated terms.
Based on the dynamically constructed debiasing list, the model trained with the staged correction process achieves similar performance in terms of fairness to the baseline, which applies the strict restrictions for debiasing through the whole training process.
Moreover, our debiased model maintains the competitive classification performance in comparison with the biased model trained under no constraints.

We also study the indirect impacts of the word-level correction on the semantically connected words, which are underexplored in previous work.
Having the experimental results supporting our assumption about the existence of indirect impacts, we argue that the word list for debiasing should be minimized by excluding less related tokens to avoid introducing unexpected effects to the training target.
Additionally, the observation also strengthens the belief that debiasing is a multi-staged task through the NLP pipeline.
Bias-free embeddings from the preprocessing stage with sensitive group identifiers distant from sentimental words could reduce the risks of the unintended indirect impacts brought by bias mitigation at the fine-tuning stage into the model capacity.

% \section*{Acknowledgment}

% \section*{References}

% Please number citations consecutively within brackets \cite{b1}. The 
% sentence punctuation follows the bracket \cite{b2}. Refer simply to the reference 
% number, as in \cite{b3}---do not use ``Ref. \cite{b3}'' or ``reference \cite{b3}'' except at 
% the beginning of a sentence: ``Reference \cite{b3} was the first $\ldots$''

% Number footnotes separately in superscripts. Place the actual footnote at 
% the bottom of the column in which it was cited. Do not put footnotes in the 
% abstract or reference list. Use letters for table footnotes.

% Unless there are six authors or more give all authors' names; do not use 
% ``et al.''. Papers that have not been published, even if they have been 
% submitted for publication, should be cited as ``unpublished'' \cite{b4}. Papers 
% that have been accepted for publication should be cited as ``in press'' \cite{b5}. 
% Capitalize only the first word in a paper title, except for proper nouns and 
% element symbols.

% For papers published in translation journals, please give the English 
% citation first, followed by the original foreign-language citation \cite{b6}.

\vspace{12pt}

\end{document}